\documentclass[11pt,a4paper]{article}

\usepackage{comment}
\usepackage{float}
\usepackage{color}
\usepackage{multicol}
\usepackage{wrapfig}
\usepackage[pdftex]{graphicx}  
\usepackage{bm}
\usepackage{url}
\usepackage{colortbl}
\usepackage{tabularx}
\usepackage{fancyhdr}
\usepackage{ulem}
\usepackage{amsmath,amssymb,amsfonts,amsthm}
\usepackage{algorithmic}
\usepackage{textcomp}
\usepackage{xcolor}
\usepackage{cleveref}
\usepackage{amsthm} 
\usepackage{graphicx}
\usepackage{booktabs}
\usepackage{subcaption}  
\usepackage{enumitem}
\usepackage{caption}

\usepackage[numbers]{natbib}

\theoremstyle{definition}

\newtheorem{proposition}{Proposition}

\newtheorem{example}{Example}

\usepackage[top=30truemm,bottom=30truemm,left=25truemm,right=25truemm]{geometry}

\begin{document}
\begin{center}
{\Large Estimation of the Learning Coefficient Using Empirical Loss} \\
\vspace{2mm}
Tatsuyoshi Takio, Joe Suzuki
\\

\end{center}


\begin{abstract}
The learning coefficient plays a crucial role in analyzing the performance of information criteria, such as the Widely Applicable Information Criterion (WAIC) and the Widely Applicable Bayesian Information Criterion (WBIC), which Sumio Watanabe developed to assess model generalization ability. In regular statistical models, the learning coefficient is given by $d / 2$, where $d$ is the dimension of the parameter space. More generally, it is defined as the absolute value of the pole order of a zeta function derived from the Kullback-Leibler divergence and the prior distribution. However, except for specific cases such as reduced-rank regression, the learning coefficient cannot be derived in a closed form.

Watanabe proposed a numerical method to estimate the learning coefficient, which Imai further refined to enhance its convergence properties. These methods utilize the asymptotic behavior of WBIC and have been shown to be statistically consistent as the sample size grows.

In this paper, we propose a novel numerical estimation method that fundamentally differs from previous approaches and leverages a new quantity, "Empirical Loss," which was introduced by Watanabe. Through numerical experiments, we demonstrate that our proposed method exhibits both lower bias and lower variance compared to those of Watanabe and Imai. Additionally, we provide a theoretical analysis that elucidates why our method outperforms existing techniques and present empirical evidence that supports our findings.
\end{abstract}

\section{Introduction}
Evaluating the generalization ability and prediction performance of statistical models is crucial in various applications. Metrics such as generalization loss, empirical loss, and free energy serve as fundamental tools for this purpose, facilitating model selection and improving estimation accuracy. These metrics are closely related to the learning coefficient $\lambda$, which plays a central role in the theoretical analysis of information criteria such as the Widely Applicable Information Criterion (WAIC)~\cite{watanabe2010} and the Widely Applicable Bayesian Information Criterion (WBIC)~\cite{watanabe2013} proposed by Sumio Watanabe. These criteria provide a unified framework for evaluating statistical models and are particularly useful for addressing singular models, which commonly appear in real-world applications.

For regular statistical models, the learning coefficient is given by $\lambda=d / 2$, where $d$ is the dimension of the parameter space~\cite{watanabe2009algebraic}. However, many practical models exhibit singularities and fail to meet the standard regularity conditions required for conventional asymptotic analysis. In such cases, conventional information criteria such as AIC and BIC are no longer applicable, necessitating alternative approaches based on Watanabe's Bayesian theory. The learning coefficient $\lambda$ appears in the asymptotic expansions of free energy and generalization loss, directly influencing the accuracy of WAIC and WBIC. Therefore, precise estimation of $\lambda$ is essential for improving the reliability of these criteria.

Existing numerical methods for estimating the learning coefficient, including those proposed by Watanabe and subsequently refined by Imai, utilize the asymptotic behavior of WBIC. While these approaches are consistent, they suffer from limitations such as high bias and variance in small-sample scenarios. In this paper, we propose a novel estimation method that fundamentally differs from previous approaches and leverages a new quantity, "Empirical Loss," introduced by Watanabe~\cite{watanabe2010}. Through numerical experiments, we demonstrate that our method exhibits both lower bias and lower variance compared to conventional methods. Furthermore, our findings contribute to a deeper understanding of singular models and strengthen the theoretical foundation of WAIC and WBIC by offering a more accurate and computationally efficient method for estimating the learning coefficient.

\section{Motivation and Objectives of This Study}

The learning coefficient $\lambda$ plays a fundamental role in the theoretical analysis of metrics and information criteria used to assess the generalization ability of statistical models. In particular, the Widely Applicable Information Criterion (WAIC) and the Widely Applicable Bayesian Information Criterion (WBIC), proposed by Sumio Watanabe, provide a unified framework applicable to both regular and singular models. Accurate estimation of $\lambda$ is crucial for understanding the asymptotic properties of WAIC and WBIC and improving their practical reliability.

Numerical methods for estimating the learning coefficient include Watanabe's method~\cite{watanabe2013} and Imai's method~\cite{imai2019}. These methods employ Markov Chain Monte Carlo (MCMC) techniques to estimate the posterior distribution of parameters and to evaluate generalization loss, including WBIC. However, these methods face several challenges, including:

\subsection*{Challenges in Existing Methods}
\begin{itemize}
    \item \textbf{Watanabe's method}\\ This method requires selecting multiple inverse temperature values $\beta$, yet no established criterion exists for determining the optimal $\beta$. As a result, the estimation accuracy is highly sensitive to $\beta$, potentially introducing bias and necessitating manual tuning, thereby increasing computational cost.
    \item \textbf{Imai's method}\\ Although Imai's method improves upon Watanabe's approach, preliminary experiments suggest that its estimates are highly variable. This variability may stem from fluctuations in MCMC sampling and the specific choice of posterior variance estimation. Alternative approaches could potentially achieve more stable and accurate estimation.
\end{itemize}  

\subsection*{Objectives and Contributions of This Study}
To address these issues, we propose a novel numerical estimation method based on a new quantity called "Empirical Loss"~\cite{watanabe2010}, introduced by Watanabe. This approach offers several advantages:
\begin{itemize}
    \item \textbf{Lower bias and variance}\\ Compared to existing methods, our estimator demonstrates improved stability in numerical experiments.
    \item \textbf{Reduced parameter selection burden}\\ Our approach eliminates the need for manual tuning of $\beta$, enhancing practical usability.
\end{itemize}
In this study, we develop this new estimator and demonstrate through theoretical analysis and numerical experiments that it achieves lower variance and mean squared error (MSE) compared to existing methods. Furthermore, we analyze the impact of MCMC sampling errors to improve the practical applicability of learning coefficient estimation in singular models.

\section{Learning Coefficient}
The learning coefficient $\lambda$ plays a fundamental role in evaluating the generalization ability of statistical models, particularly in the analysis of information criteria such as WAIC and WBIC. It characterizes the asymptotic behavior of the free energy, generalization loss, and other key quantities in Bayesian learning theory. This section provides a formal definition of the learning coefficient and examines its theoretical properties.
\subsection{Notation and Definition of the Learning Coefficient}
Let \(\Theta \subset \mathbb{R}^d\) denote the parameter space, and let $\mathcal{X}$ denote the sample space. For \(\theta \in \Theta\) and \(x \in \mathcal{X}\), let \(q(x)\) denote the distribution of the random variable \(X \in \mathcal{X}\) (hereafter referred to as the true distribution), \(p(x\mid\theta)\) the statistical model, and \(\varphi(\theta)\) the prior distribution. First, for a sample \(x_1,\dots,x_n \in \mathcal{X}\), we define the marginal likelihood, posterior distribution, and predictive distribution respectively by
$$
\begin{aligned}
Z_n(x_1, \dots, x_n) &:= \int_\Theta p(x_1 \mid \theta) \cdots p(x_n \mid \theta)\,\varphi(\theta)\,d\theta,\\[1mm]
p(\theta \mid x_1, \dots, x_n) &:= \frac{1}{Z_n}\,\varphi(\theta)\prod_{i=1}^n p(x_i \mid \theta),\\[1mm]
r\Bigl(x \mid x_1, \ldots, x_n\Bigr) &:= \int_{\Theta} p(x \mid \theta)\,p\Bigl(\theta \mid x_1, \ldots, x_n\Bigr)d\theta = \frac{Z_n\Bigl(x_1, \ldots, x_n, x\Bigr)}{Z_n\Bigl(x_1, \ldots, x_n\Bigr)}.
\end{aligned}
$$
In general, the posterior expectation of a function \(s:\Theta\rightarrow \mathbb{R}\) given the sample \(x_1,\ldots,x_n\) is written as
\begin{align}
\mathcal{E}_\theta\Bigl[s(\theta) \mid x_1,\dots,x_n\Bigr] := \int_\Theta s(\theta)\,p(\theta \mid x_1, \dots, x_n)\,d\theta \label{E}
\end{align}
For example,
$$
r\Bigl(x \mid x_1, \ldots, x_n\Bigr)=\mathcal{E}_\theta\Bigl[p(x\mid\theta) \mid x_1,\dots,x_n\Bigr].
$$
Furthermore, the empirical loss employed in our method, described later, is defined as
\begin{equation}\label{tn}
T_n := \frac{1}{n}\sum_{i=1}^n\Bigl\{-\log r\Bigl(x_i \mid x_1, \ldots, x_n\Bigr)\Bigr\}.
\end{equation}
WAIC is defined using empirical loss. Moreover, if the sample average in the definition of empirical loss is replaced by the expectation with respect to \(X\), we obtain the generalization loss
$$
G_n := \mathbb{E}_X\Bigl[-\log r\Bigl(X \mid x_1, \ldots, x_n\Bigr)\Bigr],
$$
which is the measure used to justify WAIC. In addition, the free energy
\begin{align}
F_n := -\log Z_n(x_1, \dots, x_n) \label{Fn}
\end{align}
is closely related to WBIC (Widely Applicable Bayesian Information Criterion), satisfying
$$
F_n = WBIC_n\Bigl(\frac{1}{\log n}\Bigr) + O_P\Bigl(\sqrt{\log n}\Bigr)
$$
(see~\cite{watanabe2013}). Furthermore, let the Kullback–Leibler divergence \(K:\Theta\rightarrow \mathbb{R}\), defined by the true distribution and the statistical model, be given by
$$
K(\theta) := \int_\mathcal{X} \log \Bigl(\frac{q(x)}{p(x \mid \theta)}\Bigr) q(x)\,dx,
$$
and define the corresponding zeta function as
$$
\zeta(z) := \int_\Theta K(\theta)^z\,\varphi(\theta)\,d\theta.
$$
The learning coefficient \(\lambda\) is defined as absolute order of the highest pole of this zeta function. Thus, the learning coefficient is determined by the triplet comprising the prior distribution \(\varphi(\cdot)\), the true distribution \(q(\cdot)\), and the statistical model \(p(\cdot\mid\cdot)\).

\begin{example}
Consider the setting in which \(p(\cdot \mid \theta)\) and \(q(\cdot)\) follow the normal distributions \(N(\theta, 1)\) and \(N(0, 1)\), respectively, and suppose the prior distribution \(\varphi(\cdot)\) is given by
$$
\varphi(\theta)= \begin{cases}
\frac{1}{2}, & \text{if } |\theta| \leq 1, \\
0, & \text{otherwise}.
\end{cases}
$$
In this case, the Kullback–Leibler divergence \(K(\theta)\) is expressed as
$$
K(\theta) = \int_{-\infty}^{\infty} \frac{1}{\sqrt{2\pi}} \exp\Bigl(-\frac{x^2}{2}\Bigr) \Bigl\{\frac{(x-\theta)^2}{2} - \frac{x^2}{2}\Bigr\} dx = \frac{\theta^2}{2},
$$
and consequently, the zeta function becomes
$$
\zeta(z) = \int_{-1}^{1} \Bigl(\frac{\theta^2}{2}\Bigr)^z \frac{1}{2}\,d\theta = \frac{1}{2^z (2z+1)} \cdot \frac{1 - (-1)^{2z+1}}{2}.
$$
Therefore, \(\lambda = 1/2\).
\end{example}

\subsection{Learning Coefficient in Regular and Singular Cases}
Let the optimal parameter set \(\Theta_*\) as the set of parameters \(\theta \in \Theta\) that minimize \(K(\theta)\). The true distribution \(q\) and the statistical model \(\{p(\cdot \mid \theta)\}_{\theta \in \Theta}\) are said to be in a regular relationship if the following three conditions hold:
\begin{enumerate}
    \item There exists a unique \(\theta_* \in \Theta_*\) such that \(\Theta_* = \{\theta_*\}\).
    \item The matrix
    \[
    \left[\left.\frac{\partial^2 K(\theta)}{\partial \theta_i \partial \theta_j}\right|_{\theta=\theta_*}\right] \in \mathbb{R}^{d \times d}
    \]
    is positive definite.
    \item There exists an open set \(\tilde{\Theta} \subseteq \Theta\) containing \(\theta_*\).
\end{enumerate}
Furthermore, under regularity, the following fact is known:

\begin{proposition}
Assume that the statistical model and the true distribution are in a regular relationship and realizable. Moreover, if \(\varphi(\theta_*) > 0\), then
$$
\lambda = \frac{d}{2}
$$
holds~\cite{watanabe2009algebraic}.
\end{proposition}

Here, the true distribution \(q\) is considered realizable under the statistical model \(\{p(\cdot \mid \theta)\}_{\theta \in \Theta}\) if there exists some \(\theta \in \Theta\) such that \(p(\cdot \mid \theta)=q(\cdot)\). In Watanabe's Bayesian theory the framework is developed under a more generalized condition than regularity, namely that of having relatively finite variance. Specifically, the statistical model \(\{p(\cdot \mid \theta)\}_{\theta \in \Theta}\) is said to have relatively finite variance with respect to the true distribution \(q\) if, for every \(\theta_* \in \Theta_*\), there exists a constant \(c > 0\) such that for any \(\theta \in \Theta\)
$$
\mathbb{E}_X\left[\left\{\log \frac{p\left(X \mid \theta_*\right)}{p(X \mid \theta)}\right\}^2\right] \leq c\,\mathbb{E}_X\left[\log \frac{p\left(X \mid \theta_*\right)}{p(X \mid \theta)}\right]
$$
holds. Moreover, if the true distribution \(q\) and the statistical model \(\{p(\cdot \mid \theta)\}_{\theta \in \Theta}\) are in a regular relationship, then the model is known to have relatively finite variance, and if it is realizable, it also has relatively finite variance. These three relationships are illustrated in Figure~\ref{image0}.
\begin{figure}[htbp]
\centering
\begin{minipage}[b]{0.4\columnwidth}
    \centering
    \includegraphics[width=0.95\columnwidth]{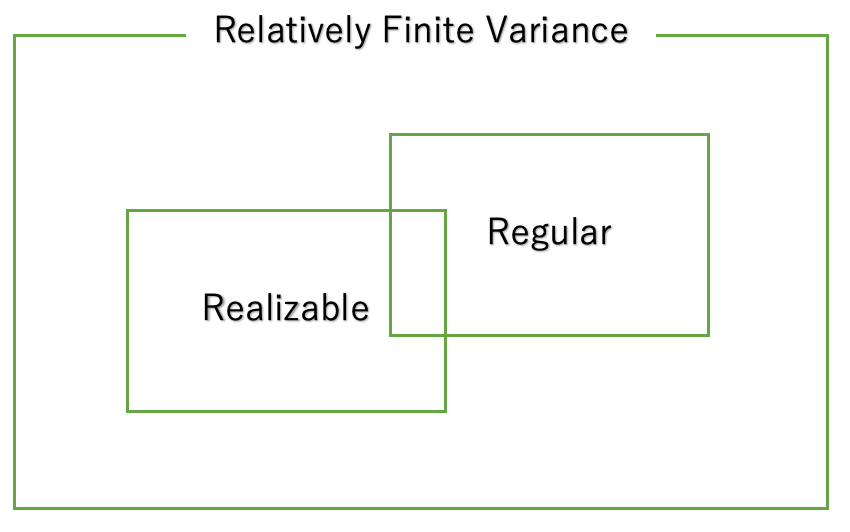}
\end{minipage}
    \caption{}
    \label{image0}
\end{figure}
Here, a model that is not regular but has relatively finite variance is referred to as a singular model. We now present several examples in which the learning coefficient for singular models has been determined.

\begin{example}
Suppose that the true distribution \(q(\cdot)\) follows \(N(0, 1)\) and that the statistical model \(p(\cdot \mid \alpha,\mu_1,\mu_2)\) is given by the two-component Gaussian mixture
\[
\alpha\,N\bigl(\mu_1, 1\bigr) + (1-\alpha)\,N\bigl(\mu_2, 1\bigr).
\]
In this case, the model is singular and the learning coefficient \(\lambda\) is \(3/4\)~\cite{Aoyaginormal}.
\end{example}

\begin{example}
Consider a mixture Poisson distribution \(f\) with \(m\) components defined using positive constants \(\lambda_i\) \((i = 1,\dots, m)\) and mixture proportions \(\theta_i\) \((0<\theta_i<1,\ \sum_{i=1}^m \theta_i=1)\) as
\[
f(x\mid \theta_1,\dots,\theta_m,\lambda_1,\dots,\lambda_m) := \sum_{i=1}^{m} \theta_i\,\mathbf{Po}\Bigl(x \mid \lambda_i\Bigr),\quad x\in \mathbb{Z}_{\geq 0}.
\]
Let the true distribution \(q\) be a mixture with \(r\) components and let the statistical model \(p\) be a mixture with \(H\) components. Then the learning coefficient for this model is given by
\[
\frac{3r + H - 2}{4},
\]
as shown in~\cite{mixpoisson}.
\end{example}

\section{Main Methods for Estimating the Learning Coefficient}
This section presents key methods for estimating the learning coefficient \(\lambda\) using WBIC. We first provide a theoretical foundation for WBIC and then describe the estimation methods proposed by Watanabe and Imai.
\subsection{WBIC}
The posterior distribution at inverse temperature $\beta>0$ is given by
\[
p_\beta(\theta \mid x_1, \dots, x_n) \;:=\; \frac{\varphi(\theta)\prod_{i=1}^n p(x_i \mid \theta)^\beta}{\displaystyle \int_\Theta \varphi(\theta)\prod_{i=1}^n p(x_i \mid \theta)^\beta\,d\theta}\,.
\]
For any function \(s:\Theta \to \mathbb{R}\), we define its posterior expectation and posterior variance at inverse temperature \(\beta>0\) as
\[
\mathcal{E}^\beta_\theta\Bigl[s(\theta) \mid x_1,\dots,x_n\Bigr] \;:=\; \int_\Theta s(\theta)\,p_\beta(\theta \mid x_1,\dots,x_n)\,d\theta,
\]
\[
\mathcal{V}^\beta_\theta\Bigl[s(\theta) \mid x_1,\dots,x_n\Bigr] \;:=\; \mathcal{E}^\beta_\theta\Bigl[s(\theta)^2 \mid x_1,\dots,x_n\Bigr] - \Bigl(\mathcal{E}^\beta_\theta\Bigl[s(\theta) \mid x_1,\dots,x_n\Bigr]\Bigr)^2\,.
\]
Furthermore, WBIC at inverse temperature $\beta>0$ is defined as
\[
WBIC_n(\beta) \;:=\; \mathcal{E}^\beta_\theta\Bigl[-\sum_{i=1}^n \log p(x_i \mid \theta)\Bigr]\,.
\]
The following asymptotic property is fundamental for estimating $\lambda$ ~\cite{watanabe2013}.

\begin{proposition}\label{WBIC}
Let \(\beta_0>0\) be an arbitrary constant. Then, when the inverse temperature is set to \(\beta = \beta_0/\log n\), it holds that
\[
WBIC_n(\beta) \;=\; \sum_{i=1}^n\bigl[-\log p(x_i \mid \theta_*)\bigr] \;+\; \frac{\lambda \log n}{\beta_0} \;+\; U_n \sqrt{\frac{\lambda \log n}{2 \beta_0}} \;+\; O_p(1)\,.
\]
Here, \(E_{X_1,\dots,X_n}[U_n] = 0\) and \(U_n\) converges in distribution to \(N(0,v)\) as \(n \to \infty\), where \(v > 0\) is a positive constant.
\end{proposition}

The methods for estimating the learning coefficient proposed by Watanabe and Imai are based on this asymptotic property of WBIC.

\subsection{Watanabe's Method}
\begin{proposition}
Let \(\beta_{01}, \beta_{02} > 0\) be constants, and define $\beta_1 := \beta_{01}/\log n, \beta_2 := \beta_{02}/\log n$. Then,
\[
\hat{\lambda}_W := \frac{WBIC_n(\beta_1) - WBIC_n(\beta_2)}{1/\beta_1 - 1/\beta_2}
\]
is a consistent estimator of the learning coefficient \(\lambda\)~\cite{watanabe2013}.
\end{proposition}

\begin{proof}
From Proposition~\ref{WBIC} we have
\[
WBIC_n(\beta_1) = \sum_{i=1}^n \bigl[-\log p(x_i \mid \theta_*)\bigr] + \frac{\lambda}{\beta_1} + O_P\bigl(\sqrt{\log n}\bigr),
\]
and
\[
WBIC_n(\beta_2) = \sum_{i=1}^n \bigl[-\log p(x_i \mid \theta_*)\bigr] + \frac{\lambda}{\beta_2} + O_P\bigl(\sqrt{\log n}\bigr).
\]
Subtracting these equations and dividing by \(1/\beta_1 - 1/\beta_2\) yields
\[
\frac{WBIC_n(\beta_1) - WBIC_n(\beta_2)}{1/\beta_1 - 1/\beta_2} = \lambda + O_P\!\left(\frac{1}{\sqrt{\log n}}\right).
\]
\end{proof}

Moreover, the following proposition holds~\cite{watanabe2013}.

\begin{proposition}
If the function \(-\frac{1}{n}\sum_{i=1}^n \log p(x_i \mid \theta)\) is not constant with respect to \(\theta\), then the following statements hold:
\begin{enumerate}
    \item \(WBIC_n(\beta)\) is a decreasing function of \(\beta\).
    \item There exists a unique \(\beta^*\) (with \(0 < \beta^* < 1\)) satisfying
    \[
    F_n = WBIC(\beta^*),
    \]
\end{enumerate}
where \(F_n\) denotes the free energy (see equation (\ref{Fn})). Furthermore, the parameter \(\beta^*\) satisfies
\[
\beta^* = \frac{1}{\log n} + o_p\!\left(\frac{1}{\log n}\right).
\]
Here, the error term \(o_p(1/\log n)\) depends on the model, the prior, and the true distribution, whereas the leading term is model-independent.
\end{proposition}

\subsection{Imai's Method}\label{imaisec}
\begin{proposition}
Let \(c>0\) be a constant and set \(\beta = c/\log n\). Then,
\[
\hat{\lambda}_I = \beta^2\,\mathcal{V}_\theta^\beta\!\left[\sum_{i=1}^n \log p(x_i \mid \theta)\right]
\]
is a consistent estimator of the learning coefficient \(\lambda\)~\cite{imai2019}.
\end{proposition}

\begin{proof}
In Watanabe's method, let \(\beta_1 = \beta\) and \(\beta_2 = \beta + \Delta \beta\). Multiplying both the numerator and the denominator by \(\beta(\beta+\Delta\beta)\) yields
\[
-\beta(\beta+\Delta\beta)\frac{WBIC_n(\beta+\Delta\beta) - WBIC_n(\beta)}{\Delta \beta} = \lambda + O_P\!\left(\frac{1}{\sqrt{\log n}}\right).
\]
Taking the derivative in the limit as $\Delta \beta \rightarrow 0$ gives
\begin{equation}\label{imai}
-\beta^2 \frac{\partial WBIC_n(\beta)}{\partial \beta} = \lambda + O_P\!\left(\frac{1}{\sqrt{\log n}}\right).
\end{equation}
On the other hand, differentiating
\[
WBIC_n(\beta) = \frac{\displaystyle \int_\Theta -\sum_{i=1}^n \log p(x_i \mid \theta)\,\varphi(\theta)\prod_{i=1}^n p(x_i \mid \theta)^\beta\,d\theta}{\displaystyle \int_\Theta \varphi(\theta)\prod_{i=1}^n p(x_i \mid \theta)^\beta\,d\theta}
\]
with respect to \(\beta\) yields
\begin{equation}\label{aaa}
\frac{\partial WBIC_n(\beta)}{\partial \beta} = -\mathcal{E}^\beta_\theta\!\left[\Bigl\{\sum_{i=1}^n \log p(x_i \mid \theta)\Bigr\}^2\right] + \Biggl(\mathcal{E}^\beta_\theta\!\left[\sum_{i=1}^n \log p(x_i \mid \theta)\right]\Biggr)^2 = -\mathcal{V}^\beta_\theta\!\left[\sum_{i=1}^n \log p(x_i \mid \theta)\right].
\end{equation}
Substituting \eqref{aaa} into \eqref{imai} gives
\[
\beta^2\,\mathcal{V}^\beta_\theta\!\left[\sum_{i=1}^n \log p(x_i \mid \theta)\right] = \lambda + O_P\!\left(\frac{1}{\sqrt{\log n}}\right).
\]
Thus, \(\hat{\lambda}_I\) is a consistent estimator of \(\lambda\).
\end{proof}

\subsubsection*{Comparison between Watanabe's and Imai's Methods.}
\begin{itemize}
    \item \textbf{watanabe's Method}\\
    Simple to implement, but requires the selection of $\beta_1, \beta_2$, and suffers from numerical instability when $\Delta \beta \rightarrow 0$.
    \item \textbf{imai's Method}\\
    While it resolves the $\Delta \beta \rightarrow 0$ issue and requires fewer parameter choices, it remains unstable with respect to variance.
\end{itemize}

\section{Proposal of an Estimator Using Empirical Loss}
In this section, we propose an estimator for the learning coefficient by the estimator
\[
\hat{\lambda}_{T} = \frac{WBIC_n - n\,T_n}{\log n}\,,
\]
The empirical loss \(T_n\) is closely related to the generalization loss and offers an alternative measure of model complexity. While WBIC captures asymptotic free energy behavior, it does not directly reflect model generalization. By incorporating \(T_n\), we introduce an regularization effect that stabilizes the estimation of \(\lambda\). This method avoids the selection of multiple inverse temperatures, as in Watanabe's method, and reduces variance instability, as seen in Imai's method.

\begin{proposition}
\(\hat{\lambda}_{T}\) is a consistent estimator of the learning coefficient \(\lambda\).
\end{proposition}

\begin{proof}
For a function \(F(t,u)\), we define the posterior expectation with respect to the “integrated” posterior distribution as
\[
\mathcal{E}_{t,u}\Bigl[F(t,u)\,\Big|\,\xi_n\Bigr] := \frac{\displaystyle \int du^* \int_0^\infty dt\, F(t,u)\,t^{\lambda-1}\exp\Bigl(-t+\sqrt{t}\,\xi_n(u)\Bigr)}{\displaystyle \int du^* \int_0^\infty dt\, t^{\lambda-1}\exp\Bigl(-t+\sqrt{t}\,\xi_n(u)\Bigr)}\,.
\]
Similarly, we define the corresponding posterior variance by
\[
\mathcal{V}_{t,u}\Bigl[F(t,u)\,\Big|\,\xi_n\Bigr] := \mathcal{E}_{t,u}\Bigl[F(t,u)^2\,\Big|\,\xi_n\Bigr] - \Bigl(\mathcal{E}_{t,u}\Bigl[F(t,u)\,\Big|\,\xi_n\Bigr]\Bigr)^2\,.
\]
(For further details, see~\cite{watanabe2018}.) Note that as \(n\to\infty\), \(\mathcal{E}_{t,u}\bigl[\cdot\,|\,\xi_n\bigr]\) converges to \(\mathcal{E}_\theta\bigl[\cdot\,|\,x_1,\dots,x_n\bigr]\) (see Equation (\ref{E})). Furthermore, define the posterior expectation and posterior variance of the negative log-likelihood by
\[
\mathcal{E}(x) := \mathcal{E}_{t,u}\Bigl[-\log p\bigl(x \mid g(u)\bigr)\,\Big|\,\xi_n\Bigr]\quad\text{and}\quad \mathcal{V}(x) := \mathcal{V}_{t,u}\Bigl[-\log p\bigl(x \mid g(u)\bigr)\,\Big|\,\xi_n\Bigr]\,.
\]
The empirical loss admits the following asymptotic expansion~\cite{watanabe2018}:
\begin{equation}\label{Tn1}
n\,T_n = \sum_{i=1}^n \bigl[-\log p(x_i \mid \theta_*)\bigr] + \lambda - \frac{1}{2}\,\mathcal{E}_{t,u}\Bigl[\sqrt{t}\,\xi_n(u) \,\Big|\,\xi_n\Bigr] - \frac{1}{2}\sum_{i=1}^n \mathcal{V}(x_i) + o_P(1)\,.
\end{equation}
Moreover, by Proposition~\ref{WBIC}, WBIC admits the following expansion
\begin{equation}\label{WBIC2}
WBIC_n = \sum_{i=1}^n \bigl[-\log p(x_i \mid \theta_*)\bigr] + \lambda\,\log n + O_P\bigl(\sqrt{\log n}\bigr)\,.
\end{equation}
Noting that the term
\[
\lambda - \frac{1}{2}\,\mathcal{E}_{t,u}\Bigl[\sqrt{t}\,\xi_n(u) \,\Big|\,\xi_n\Bigr] - \frac{1}{2}\sum_{i=1}^n \mathcal{V}(x_i) + o_P(1)
\]
in Equation (\ref{Tn1}) is \(O_P(1)\), subtracting Equation (\ref{Tn1}) from Equation (\ref{WBIC2}) yields
\[
WBIC_n - n\,T_n = \lambda\,\log n + O_P\bigl(\sqrt{\log n}\bigr)\,.
\]
Dividing both sides by \(\log n\) then gives
\[
\frac{WBIC_n - n\,T_n}{\log n} = \lambda + O_P\!\Bigl(\frac{1}{\sqrt{\log n}}\Bigr)\,.
\]
Consequently, \(\hat{\lambda}_{T}\) is a consistent estimator of the learning coefficient \(\lambda\).
\end{proof}

\section{Numerical Experiments}\label{suuchi}
\subsection{Analysis of Bias and Variance for Each Method}\label{suuchi1}
We conducted experiments under several settings to compare the performance of Watanabe's method, Imai's method, and the proposed method. For the regular model, we used:
\begin{itemize}
    \item True distribution: the normal distribution \(N(0,1)\).
    \item Statistical model: the normal distribution \(N(\mu,\sigma^2)\).
    \item For each sample size (as shown on the horizontal axis), 500 estimates of the learning coefficient were computed and then averaged.
\end{itemize}
(see Figure~\ref{image1}, left). For the singular model, we used:
\begin{itemize}
    \item True distribution: the Poisson distribution \(Po(3)\).
    \item Statistical model: a two-component mixture Poisson distribution defined by 
    \[
    p(\cdot \mid \theta,\lambda_1,\lambda_2) = \theta\,Po(\lambda_1) + (1-\theta)\,Po(\lambda_2)\,,
    \]
    \item For each sample size (as shown on the horizontal axis), 300 estimates of the learning coefficient were generated and averaged.
\end{itemize}
(see Figure~\ref{image1}, right). Figure~\ref{image1} shows the simulation results.
\clearpage
\begin{figure}[htbp]
\centering
\begin{minipage}[b]{0.49\columnwidth}
    \centering
    \includegraphics[width=0.8\columnwidth]{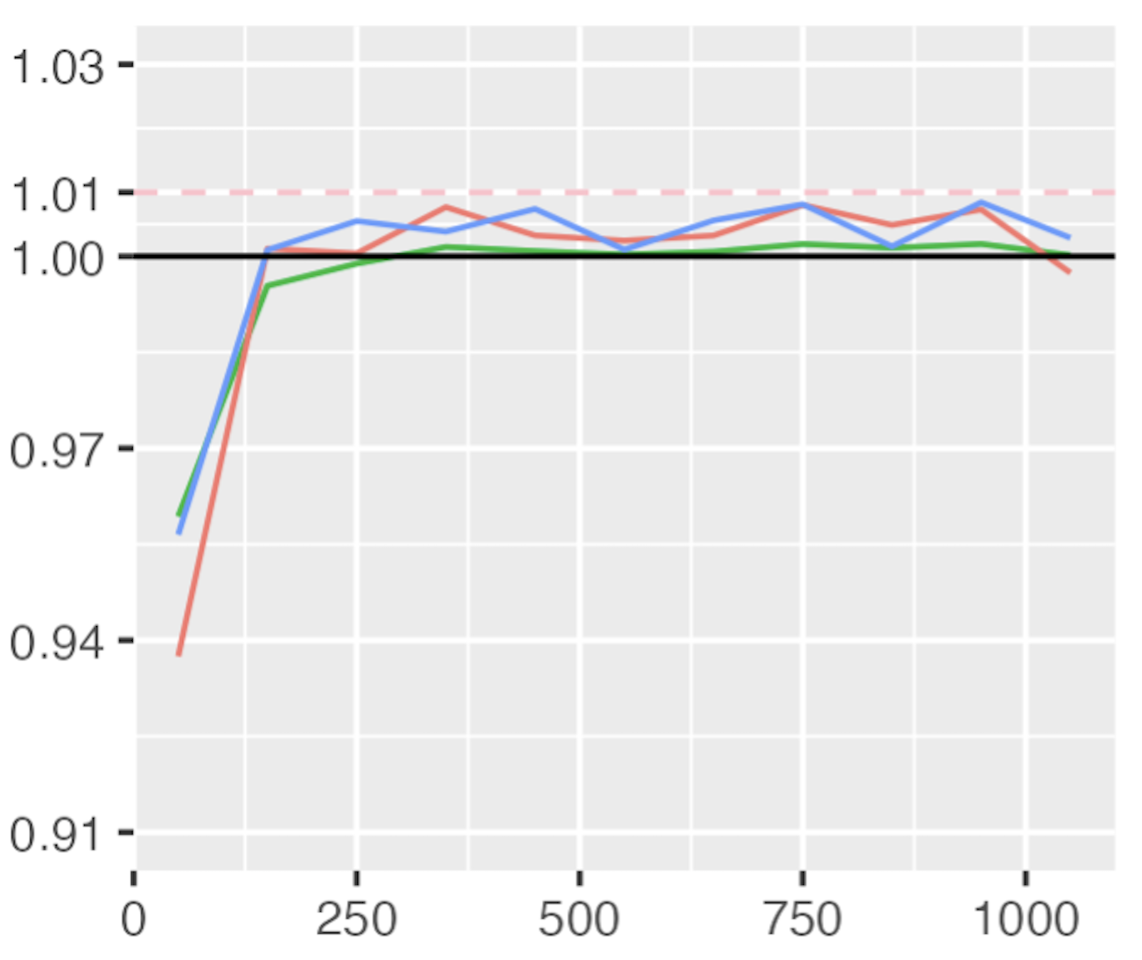}
\end{minipage}
\begin{minipage}[b]{0.49\columnwidth}
    \centering
    \includegraphics[width=0.95\columnwidth]{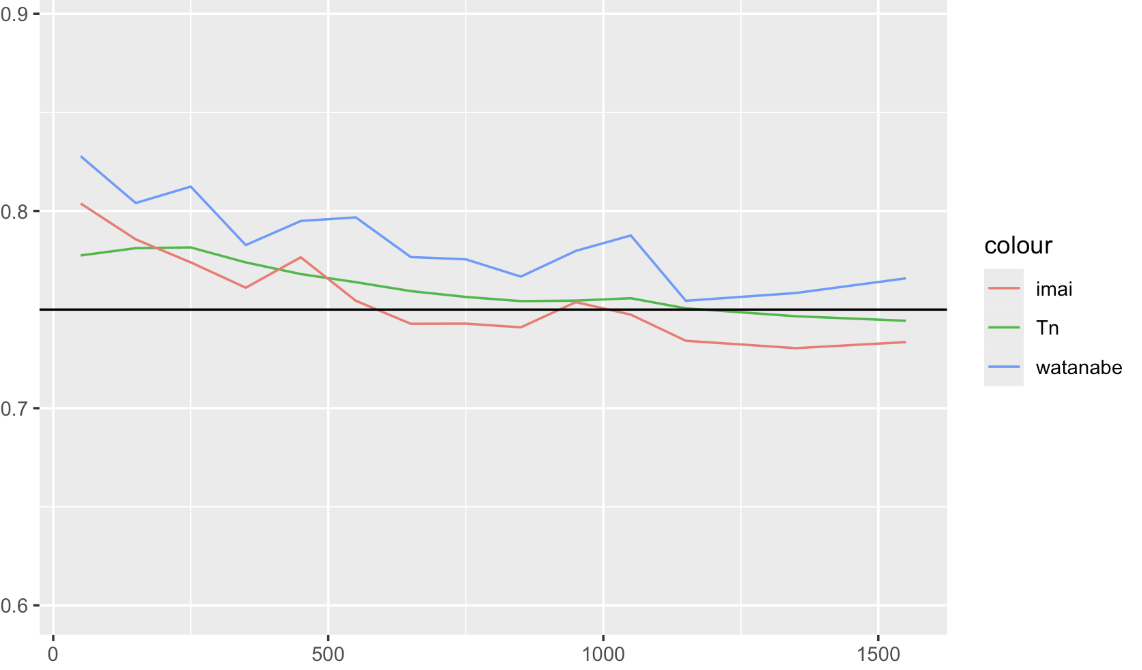}
\end{minipage}
\caption{Graph of learning coefficient estimates versus sample size (horizontal axis) for each method. Left: Example of a regular model using a normal distribution; Right: Example of a singular model using a two-component mixture Poisson distribution.}
\label{image1}
\end{figure}

Table~\ref{table1} shows that the estimates of the learning coefficient obtained using the normal and Poisson distributions closely approximate the true learning coefficient across Watanabe\textquoteright s method, Imai\textquoteright s method, and the proposed method. Furthermore, in the Poisson case (Figure~\ref{image1}, right) the proposed method shows a significant reduction in variance. To verify this, we set the sample size to 750 and compiled a table summarizing the means and variances (Table~\ref{table1}). In this experiment, the true distribution follows a Poisson distribution with mean 3, and the statistical model is a two-component mixture Poisson distribution. The true learning coefficient is 0.75.
\begin{table}[htbp]
    \centering
    \begin{tabular}{cccccc}
        \toprule
         & \(\hat{\lambda}_W(1,1.5)\) & \(\hat{\lambda}_W(1,3)\) & \(\hat{\lambda}_W(1,5)\) & \(\hat{\lambda}_W(1,1000)\) \\
        \midrule
        Mean   & 0.7766225 & 0.7727747 & 0.7732914 & 0.8082598 \\
        Bias   & 0.0266225 & 0.0227747 & 0.0232914 & 0.0582598 \\
        Variance & 0.0529603 & 0.0112726 & 0.0074551 & \textcolor{red}{0.0071719} \\
        MSE    & 0.05356317 & 0.06675754 & 0.01294827 & 0.01360033 \\
    \end{tabular}
    \bigskip
    \begin{tabular}{cccc}
        \toprule
         & \(\hat{\lambda}_{I}\) & \(\hat{\lambda}_{T}\)  \\
        \midrule
        Mean   & \textcolor{red}{0.7540509} & \textcolor{red}{0.7580027}  \\
        Bias   & \textcolor{red}{0.0040509} & \textcolor{red}{0.0080027}  \\
        Variance & 0.0172270 & \textcolor{red}{0.0043311}  \\
        MSE    & 0.01720898 & \textcolor{red}{0.00438658} \\
        \bottomrule
    \end{tabular}
    \caption{In an experiment in which the true distribution follows a Poisson distribution with mean 3 and the statistical model is a two-component mixture Poisson distribution, the learning coefficient was estimated and the Mean, Bias, Variance, and MSE were measured. In the upper table, Watanabe\textquoteright s method is presented with two inverse temperature values \(\beta_1\) and \(\beta_2\) defined as \(\hat{\lambda}_W(\beta_{01},\beta_{02})\) with \(\beta_1 = \beta_{01}/\log n\) and \(\beta_2 = \beta_{02}/\log n\).}
    \label{table1} 
\end{table}

Table~\ref{table1} indicates that the estimators with the lowest bias are \(\hat{\lambda}_{I}\) and \(\hat{\lambda}_{T}\), and that $\hat{\lambda}_{T}$ exhibits the lowest variance and MSE. Moreover, it is evident that for \(\hat{\lambda}_W\) the variance decreases as the gap between the \(\beta\) values increases, whereas if the gap becomes excessively large, the bias increases. From these experiments, it appears that estimating the learning coefficient using the empirical loss \(T_n\) has a particular advantage in terms of variance. Below, we discuss the reasons for this. The proposed estimator
\[
\hat{\lambda}_{T} = \frac{WBIC_n - \textcolor{red}{n\,T_n}}{\log n}
\]
differs fundamentally from Watanabe\textquoteright s method,
\[
\hat{\lambda}_W = \frac{WBIC_n(\beta_1) - \textcolor{red}{WBIC_n(\beta_2)}}{1/\beta_1 - 1/\beta_2}\,,
\]
in that the second WBIC in the numerator is replaced by the empirical loss. A possible explanation for the smaller variance of \(\hat{\lambda}_{T}\) is that the variance of the empirical loss \(n\,T_n\) is lower than that of WBIC itself. To test this, we measured the variances of WBIC and the empirical loss, as shown in Table~\ref{table2}.

\begin{table}[htbp]
    \centering
    \scalebox{0.70}{
    \begin{tabular}{ccccc}
        \toprule
         Variance of \(n\,T_n\) & Variance of \(WBIC\Bigl(\frac{1}{\log n}\Bigr)\) & Variance of \(WBIC\Bigl(\frac{1.5}{\log n}\Bigr)\) & Variance of \(WBIC\Bigl(\frac{3}{\log n}\Bigr)\) & Variance of \(WBIC\Bigl(\frac{1000}{\log n}\Bigr)\) \\
        \midrule
         314.4476 & 312.6105 & 315.236 & 317.9765 & 315.8784 \\
        \bottomrule
    \end{tabular}
    }
    \caption{Measured variances of the empirical loss and WBIC.}
    \label{table2}
\end{table}

However, Table~\ref{table2} reveals no significant difference between the variances of WBIC and the empirical loss itself. Next, under the same conditions as in Table~\ref{table1}, we computed
\setlength{\jot}{10pt}
\begin{align*}
         V[\hat{\lambda}_W] &= V\left[\frac{WBIC_n\left(\beta_1\right)-WBIC_n\left(\beta_2\right)}{1 / \beta_1-1 / \beta_2}\right] \\
         &= \frac{V[WBIC_n(\beta_1)]-2cov(WBIC_n(\beta_1), WBIC_n(\beta_2)) + V[WBIC_n(\beta_2)]}{(1 / \beta_1-1 / \beta_2)^2}\\
        V[\hat{\lambda}_T]&=V\left[\frac{WBIC_n-n T_n}{\log n}\right] \\
        &= \frac{V[WBIC_n] - 2cov(WBIC_n, n T_n) + V[n T_n]}{(\log n)^2}
\end{align*}

Table~\ref{table3} presents the numerator and denominator values for the variances of \(\hat{\lambda}_W\) and \(\hat{\lambda}_T\) computed independently.

\begin{table}[htbp]
    \centering
    \begin{tabular}{ccccccc}
        \toprule
         Variance & \(\hat{\lambda}_W(1,1.5)\) & \(\hat{\lambda}_W(1,3)\) & \(\hat{\lambda}_W(1,5)\) & \(\hat{\lambda}_T\)  \\
        \midrule
         Numerator Value & 0.318103 & 0.268643 & 0.260338 & 0.269056 \\
         Denominator Value & 4.869485 & 19.47794 & 28.04824 & \textcolor{red}{43.82537} \\
         Variance \(=\) Numerator/Denominator & 0.065325 & 0.013792 & 0.009281 & 0.006139 \\
        \bottomrule
    \end{tabular}
    \caption{Separate calculations of the numerator and denominator for the variances of \(\hat{\lambda}_W\) and \(\hat{\lambda}_T\).}
    \label{table3}
\end{table}

From Table~\ref{table3}, although the numerator values do not differ significantly, the denominator for \(\hat{\lambda}_W\) increases markedly as the gap between the \(\beta\) values increases, thereby contributing to a reduction in the overall variance (variance \(=\) numerator/denominator). Moreover, the proposed estimator \(\hat{\lambda}_T\) attains the smallest variance because its denominator is the single term \(\log n\), which does not decrease. This is one plausible explanation for the reduced variance of \(\hat{\lambda}_T\).

\subsection{Relationship between Watanabe's Method and Imai's Method}
As described in Section~\ref{imaisec}, Imai's method is obtained by letting the gap between the \(\beta\) values in Watanabe's method approach zero, thereby effectively taking a derivative. Empirically, it has been confirmed that the proposed method exhibits a smaller variance than Imai's method; however, because Imai's method is essentially a derivative, it is difficult to decompose it into numerator and denominator components as in the analysis of Watanabe's method in Section~\ref{suuchi1}. Therefore, in this section we analyze Imai's method—not with a formal theoretical proof but by computing the bias and variance of the learning coefficient as a function of the gap between the \(\beta\) values in Watanabe's method.

\begin{figure}[htbp]
\centering
\begin{minipage}[b]{0.8\columnwidth}
    \centering
    \includegraphics[width=0.95\columnwidth]{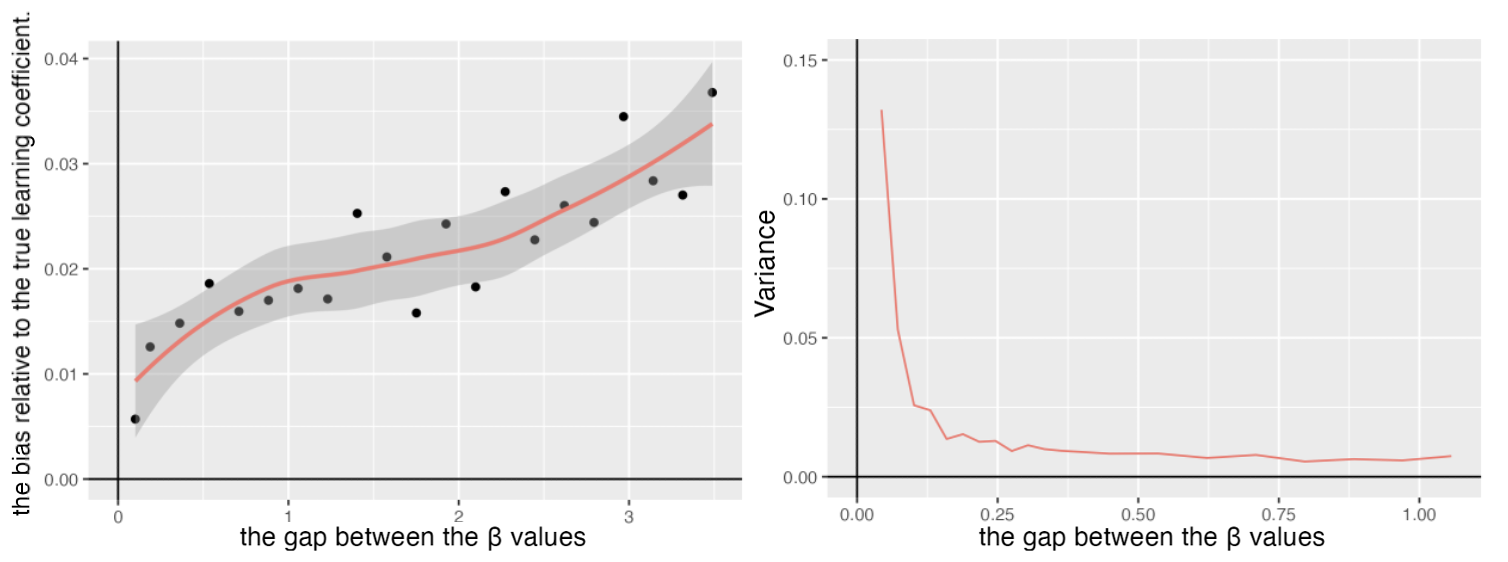}
\end{minipage}
\caption{For Watanabe's method with \(\beta_1 = 1/\log n\) and \(\beta_2 > \beta_1\), the horizontal axis represents the gap \(\beta_2-\beta_1\). The left panel shows the bias measured for each \(\beta\) gap, while the right panel shows the corresponding variance.}
\label{image5}
\end{figure}

Figure~\ref{image5} displays the learning coefficient computed when the true distribution is a Poisson distribution \(Po(3)\) and the statistical model is a two-component mixture Poisson distribution. In this experiment, using Watanabe's method we set \(\beta_1 = 1/\log n\) and allow \(\beta_2\) to exceed \(\beta_1\), with the horizontal axis representing the gap \(\beta_2-\beta_1\). The vertical axis in the left panel represents the bias corresponding to each gap, while the right panel shows the measured variance. This graph confirms that within Watanabe's method, as the gap between the \(\beta\) values decreases, the variance increases while the bias decreases. Interpreting Imai's method as the limit in which the gap tends to zero, the bias is minimized but the variance becomes large. This suggests that Imai's method is an estimator that maintains low bias at the expense of increased variance. Moreover, as indicated in Table~\ref{table1}, the proposed method exhibits nearly the same bias as Imai's method while achieving superior variance performance. These results thus suggest that the method based on empirical loss is more stable in terms of both variance and bias than both Imai's method and Watanabe's method.

\subsection{Effect of MCMC Outliers on Imai's Method and the Empirical Loss-Based Method}
MCMC sampling can produce outliers; although it is technically possible to remove such outliers, it is difficult to eliminate them entirely. Therefore, an estimator that is robust even in the presence of outliers is desirable. In this subsection, we first describe the nature of these outliers and then examine how they affect the behavior of the learning coefficient estimates produced by Imai's method and the proposed method.

The following conditions were used:
\begin{itemize}
    \item True distribution: \(N(0,1)\).
    \item Statistical model: a four-component Gaussian mixture.
    \item Number of training samples: 1500.
    \item Number of MCMC iterations: 4000.
    \item True learning coefficient: 1.25.
    \item Inverse temperature: \(\beta = 1/\log n\).
\end{itemize}

Under these conditions, while measuring the learning coefficient estimates via Imai's method and the proposed method, we obtained the following results (Table~\ref{table5}):
\begin{table}[htbp]
    \centering
    \begin{tabular}{ccc}
        \toprule
        & \(\hat{\lambda}_T\) & \(\hat{\lambda}_I\) \\
        \midrule
        183 & 1.189898 & 1.313025 \\
        184 & 1.593850 & 4.066127 \\
        185 & 1.203485 & 1.271636 \\
        \bottomrule
    \end{tabular}
    \caption{}
    \label{table5}
\end{table}

As shown in Table~\ref{table5}, at iteration 184 Imai's method produces an abnormally large learning coefficient compared to the proposed method. To investigate the cause, we examined the sum of the log-likelihoods
\[
\sum_{i=1}^n \log p(x_i \mid \theta_k)
\]
computed from the MCMC samples \(\{\theta_k\}_{k=1}^{4000}\). The posterior mean of these 4000 sums is approximately given by
\[
\mathcal{E}_\theta^{\beta=1/\log n}\!\left[\sum_{i=1}^n \log p(x_i \mid \theta)\right] \simeq \frac{1}{4000} \sum_{k=1}^{4000} \Biggl(\sum_{i=1}^n \log p(x_i \mid \theta_k)\Biggr),
\]
and the sample variance is approximated by
\[
\mathcal{V}_\theta^{\beta=1/\log n}\!\left[\sum_{i=1}^n \log p(x_i \mid \theta)\right] \simeq \frac{1}{4000} \sum_{k=1}^{4000} \left\{\sum_{i=1}^n \log p(x_i \mid \theta_k) - \frac{1}{4000} \sum_{k=1}^{4000} \Biggl(\sum_{i=1}^n \log p(x_i \mid \theta_k)\Biggr)\right\}^2.
\]
Figure~\ref{image6} displays the graph of the sum of the log-likelihoods for each iteration; the left panel shows a typical iteration (e.g., iteration 183), and the right panel displays iteration 184.
\begin{figure}[htbp]
\centering
\begin{minipage}[b]{0.8\columnwidth}
    \centering
    \includegraphics[width=0.95\columnwidth]{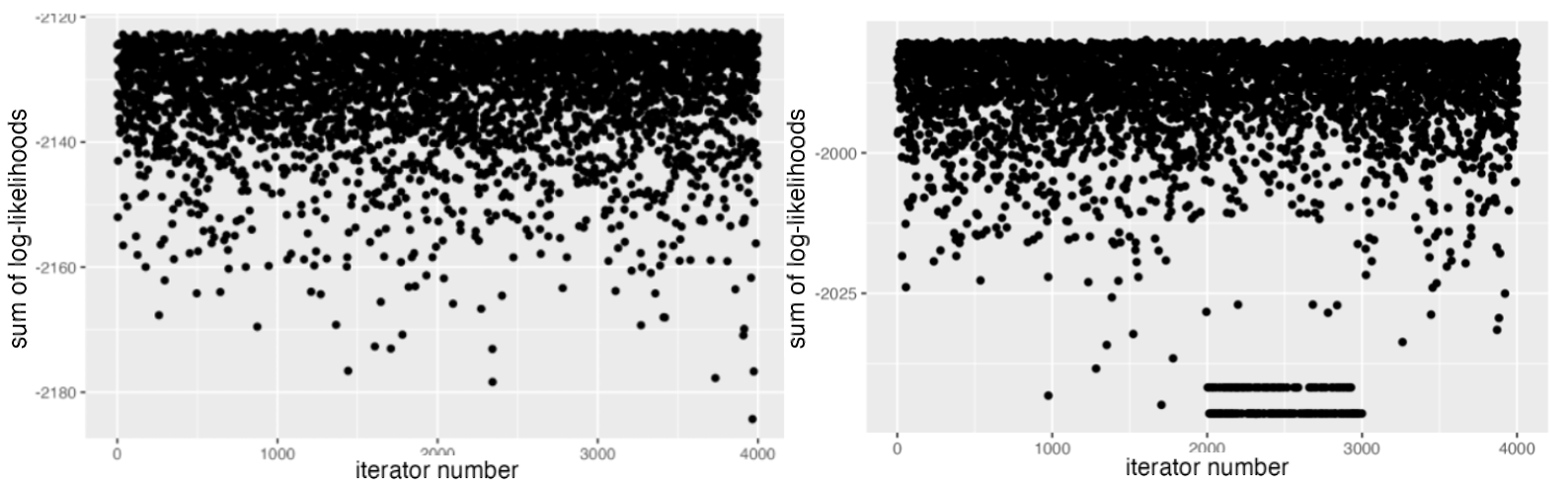}
\end{minipage}
\caption{Graph of the sum of log-likelihoods per iteration. Left: a typical iteration (e.g., iteration 183); Right: iteration 184.}
\label{image6}
\end{figure}

The abnormal value observed in Table~\ref{table5} can be attributed to anomalies in the MCMC sampling, as evidenced by Figure~\ref{image6}. Next, we investigate why, when such outliers occur, Imai's method deviates significantly from the true learning coefficient.

Figure~\ref{image8} (left panel) shows an example in which both Imai's method and the proposed method yield learning coefficient estimates close to the theoretical value, while the right panel shows the same graph with 50 artificial outlier points (red dots) added (from iteration 4001 to 4050). Figure~\ref{image7} illustrates the effect on the learning coefficient estimates as these 50 red points are gradually shifted downward. As seen in Figure~\ref{image7}, Imai's method deviates from the true learning coefficient in a quadratic manner as the outliers become more extreme, whereas the proposed method deviates only linearly. This is because the proposed method computes the average of \(\sum_{i=1}^n \log p(x_i \mid \theta)\), while Imai's method computes its variance. Consequently, the proposed method is more robust to outliers in the MCMC sampling (e.g., from Stan) than Imai's method.

\begin{figure}[htbp]
\centering
\begin{minipage}[b]{0.8\columnwidth}
    \centering
    \includegraphics[width=0.95\columnwidth]{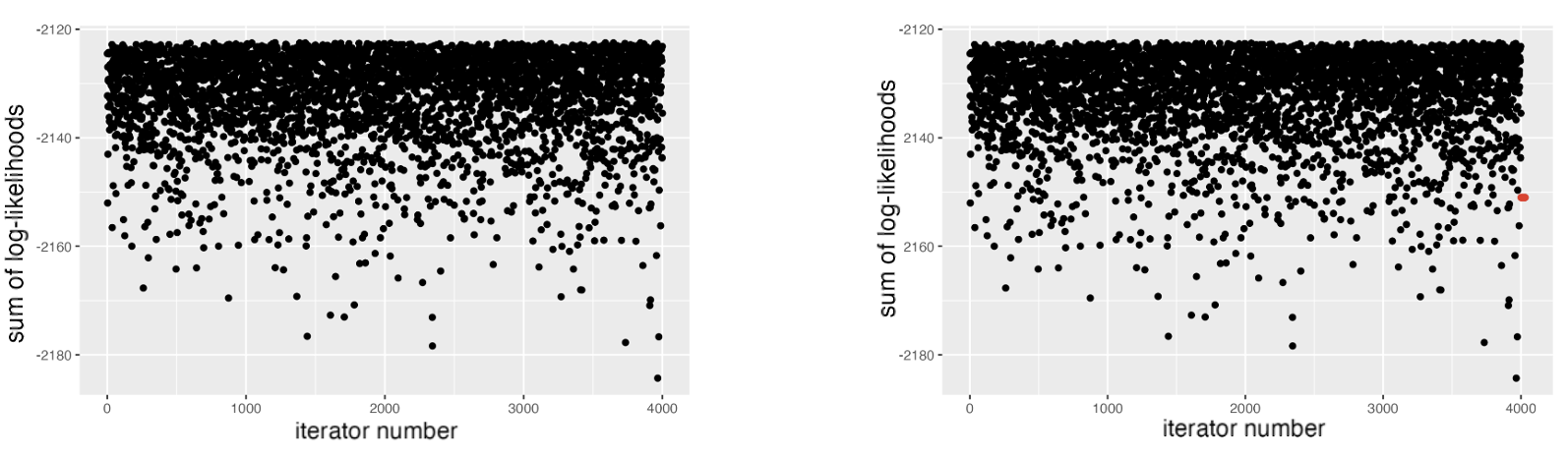}
\end{minipage}
\caption{Graph of the sum of log-likelihoods for each iteration. Left: the original graph; Right: the graph after 50 artificial outlier points (red dots) have been added (from iteration 4001 to 4050).}
\label{image8}
\end{figure}

\begin{figure}[htbp]
\centering
\begin{minipage}[b]{0.7\columnwidth}
    \centering
    \includegraphics[width=0.8\columnwidth]{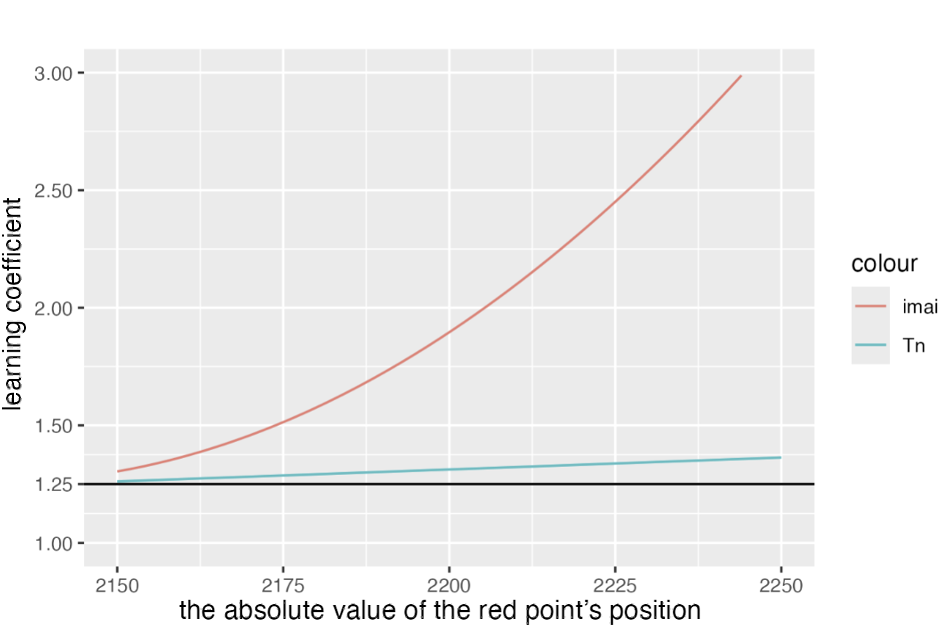}
\end{minipage}
\caption{Change in the estimated learning coefficient as the red outlier points are gradually shifted downward. The black line indicates the true learning coefficient, 1.25.}
\label{image7}
\end{figure}

\section{Conclusion}
The proposed estimator \(\hat{\lambda}_T\) does not require the tuning of the parameter \(\beta\), unlike \(\hat{\lambda}_W\). Moreover, \(\hat{\lambda}_T\) not only approximates the true learning coefficient as accurately as \(\hat{\lambda}_W\) and \(\hat{\lambda}_I\) on average, but also exhibits superior performance in terms of variance. In the comparison between \(\hat{\lambda}_W\) and \(\hat{\lambda}_T\), the former suffers from a small effective denominator due to the difference between \(1/\beta\) values, whereas the latter benefits from a stable, larger denominator consisting of the single term \(\log n\), which contributes to its lower variance. Furthermore, an analysis based on Watanabe's method suggests that Imai's method, while achieving a smaller bias, incurs a substantially larger variance. In addition, regarding the effect of outliers from MCMC sampling, \(\hat{\lambda}_I\) is affected quadratically, whereas \(\hat{\lambda}_T\) is affected only linearly.

\section{Future Work}
In this study, we have demonstrated that the learning coefficient estimator based on empirical loss exhibits superior performance in terms of both variance and bias. However, the theoretical basis for this improved performance is not yet fully elucidated, and further investigation is required to clarify the underlying reasons. Moreover, this study was conducted under a specific prior distribution; a systematic examination of the influence of different prior distributions on the estimation of the learning coefficient remains an important direction for future research.

\addcontentsline{toc}{chapter}{References} 
\renewcommand{\bibname}{References}
\bibliography{takio}
\bibliographystyle{unsrt}
    \nocite{drton2016bayesianinformationcriterionsingular}
    \nocite{Aoyaginormal}
    \nocite{mixpoisson}
    \nocite{AOYAGI2005924}
    \nocite{watanabe2013}
    \nocite{imai2019}
    \nocite{joe}
    \nocite{liu2024learning}
    \nocite{watanabe2009algebraic}
    \nocite{watanabe2010}

\end{document}